\PassOptionsToPackage{unicode}{hyperref}
\PassOptionsToPackage{hyphens}{url}
\documentclass[
]{article}
\usepackage{amsmath,amssymb}
\usepackage{lmodern}
\usepackage{iftex}
\ifPDFTeX
  \usepackage[T1]{fontenc}
  \usepackage[utf8]{inputenc}
  \usepackage{textcomp} 
\else 
  \usepackage{unicode-math}
  \defaultfontfeatures{Scale=MatchLowercase}
  \defaultfontfeatures[\rmfamily]{Ligatures=TeX,Scale=1}
\fi
\IfFileExists{upquote.sty}{\usepackage{upquote}}{}
\IfFileExists{microtype.sty}{
  \usepackage[]{microtype}
  \UseMicrotypeSet[protrusion]{basicmath} 
}{}
\makeatletter
\@ifundefined{KOMAClassName}{
  \IfFileExists{parskip.sty}{%
    \usepackage{parskip}
  }{
    \setlength{\parindent}{0pt}
    \setlength{\parskip}{6pt plus 2pt minus 1pt}}
}{
  \KOMAoptions{parskip=half}}
\makeatother
\usepackage{xcolor}
\usepackage[margin=1in]{geometry}
\usepackage{color}
\usepackage{fancyvrb}

\DefineVerbatimEnvironment{Highlighting}{Verbatim}{commandchars=\\\{\}}
\usepackage{framed}
\definecolor{shadecolor}{RGB}{248,248,248}
\newenvironment{Shaded}{\begin{snugshade}}{\end{snugshade}}

\newcommand{\AttributeTok}[1]{\textcolor[rgb]{0.77,0.63,0.00}{#1}}

\newcommand{\CommentTok}[1]{\textcolor[rgb]{0.56,0.35,0.01}{\textit{#1}}}

\newcommand{\ConstantTok}[1]{\textcolor[rgb]{0.00,0.00,0.00}{#1}}
\newcommand{\ControlFlowTok}[1]{\textcolor[rgb]{0.13,0.29,0.53}{\textbf{#1}}}

\newcommand{\DecValTok}[1]{\textcolor[rgb]{0.00,0.00,0.81}{#1}}

\newcommand{\FloatTok}[1]{\textcolor[rgb]{0.00,0.00,0.81}{#1}}
\newcommand{\FunctionTok}[1]{\textcolor[rgb]{0.00,0.00,0.00}{#1}}

\newcommand{\NormalTok}[1]{#1}

\newcommand{\OtherTok}[1]{\textcolor[rgb]{0.56,0.35,0.01}{#1}}

\newcommand{\SpecialCharTok}[1]{\textcolor[rgb]{0.00,0.00,0.00}{#1}}

\newcommand{\StringTok}[1]{\textcolor[rgb]{0.31,0.60,0.02}{#1}}

\usepackage{longtable,booktabs,array}
\usepackage{calc} 
\usepackage{etoolbox}
\makeatletter
\patchcmd\longtable{\par}{\if@noskipsec\mbox{}\fi\par}{}{}
\makeatother
\IfFileExists{footnotehyper.sty}{\usepackage{footnotehyper}}{\usepackage{footnote}}
\makesavenoteenv{longtable}
\usepackage{graphicx}
\makeatletter
\def\maxwidth{\ifdim\Gin@nat@width>\linewidth\linewidth\else\Gin@nat@width\fi}
\def\maxheight{\ifdim\Gin@nat@height>\textheight\textheight\else\Gin@nat@height\fi}
\makeatother
\setkeys{Gin}{width=\maxwidth,height=\maxheight,keepaspectratio}
\makeatletter
\def\fps@figure{htbp}
\makeatother
\usepackage{bbm}
\usepackage{booktabs}
\usepackage{longtable}
\usepackage{array}
\usepackage{multirow}
\usepackage{wrapfig}
\usepackage{float}
\usepackage{colortbl}
\usepackage{pdflscape}
\usepackage{tabu}
\usepackage{threeparttable}
\usepackage{threeparttablex}
\usepackage[normalem]{ulem}
\usepackage{makecell}
\usepackage{xcolor}
\ifLuaTeX
  \usepackage{selnolig}  
\fi
\IfFileExists{bookmark.sty}{\usepackage{bookmark}}{\usepackage{hyperref}}
\IfFileExists{xurl.sty}{\usepackage{xurl}}{} 
\hypersetup{
  pdftitle={Tight relative estimation in the mean of Bernoulli random variables},
  pdfauthor={Mark Huber},
  hidelinks,
  pdfcreator={LaTeX via pandoc}}

\title{Tight relative estimation in the mean of Bernoulli random variables}
\author{Mark Huber}
\date{}

\usepackage{amsthm}
\newtheorem{theorem}{Theorem}[section]

\theoremstyle{definition}

\theoremstyle{definition}

\theoremstyle{definition}

\theoremstyle{definition}

\theoremstyle{remark}

\begin{document}
\maketitle

\newcommand{\ind}{{\mathbbm{1}}}
\newcommand{\bbone}{{\mathbbm{1}}}
\newcommand{\bbP}{{\mathbb{P}}}
\newcommand{\bbE}{{\mathbb{E}}}

\hypertarget{abstract}{%
\paragraph*{Abstract}\label{abstract}}
\addcontentsline{toc}{paragraph}{Abstract}

Given a stream of Bernoulli random variables, consider the problem of
estimating the mean of the random variable within a specified relative error with a specified probability of failure. Until now, the Gamma Bernoulli Approximation Scheme (\textsf{GBAS}) was the method that accomplished this goal using the smallest number of average samples. In this work, a new method is introduced that is faster when the mean is bounded away from zero. The process uses a two-stage process together with some simple inequalities to get rigorous bounds on the error probability.

\hypertarget{introduction}{%
\section{Introduction}\label{introduction}}

\label{sec:1}
Suppose that \(B_1, B_2, \ldots\) are a stream of independent, identically distributed (iid) random variables that have a Bernoulli distribution. That is, for each \(i\), \({\mathbb{P}}(B_i = 1) = p\) and \({\mathbb{P}}(B_i = 0) = 1 - p\). Consider the problem of estimating \(p\).

This is one of the oldest problems in probability, and was the inspiration for the first Law of Large Numbers of Jacob Bernoulli \cite{Hu_Be1713}. There are numerous modern applications, including network analysis \cite{Hu_At2018}, testing sequence association \cite{Hu_ba1996}, testing if a population is in Hardy-Weinburg equilibrium \cite{Hu_Hu2006}, ray tracing \cite{Hu_Sh2014}, and many others \cite{Hu_Gu2002, Hu_Ar2010, Hu_Ko2008, Hu_Hu2009, Hu_Ol2010}.

The basic algorithm simply generates a fixed number of \(B_i\) values, say \(n\), and then estimates \(p\) using the sample average \(\hat p_n = (B_1 + \cdots + B_n) / n\). While both simple and unbiased, this method fails to have any guarantees on the relative error of the estimate. Taking a random number of samples can help create an estimate with such a guarantee.

The goal here will be, given parameters \(\epsilon > 0\) and \(\delta \in (0, 1)\), find an integrable stopping time \(T\) with respect to the natural filtration and an estimate
\begin{equation}
\hat p = \hat p(B_1, B_2, \ldots, B_T)
\end{equation}
such that the relative error is greater than \(\epsilon\) with probability at most \(\delta\). That is,
\begin{equation}
{\mathbb{P}}\left( \left| \frac{\hat p}{p} - 1 \right| > \epsilon \right) \leq \delta.
\end{equation}

This is also known as a \emph{randomized approximation scheme}, or \emph{ras} for short.

The time to find the estimate is dominated by the time needed to generate \(B_1, \ldots, B_T\). Hence the running time of the procedure is taken to be \({\mathbb{E}}(T)\), the average number of samples needed by the ras.

Returning to the sample average for the moment, the variance of \(\hat p_n\) is \(v_n = p(1 - p) / n\). The standard deviation is then \(\sqrt{v_n}\). To get the standard deviation to be \(\epsilon p\), \(n = (1 - p) \epsilon^{-2} p^{-1}\) samples are necessary. Using this together with the mean of medians technique \cite{Hu_JeVaVa1986, Hu_JeSi1989}, it is possible to get an ras for this sample average using
\begin{equation}
\Theta((1 - p)p^{-1} \epsilon^{-2} \ln(\delta^{-1})) 
\end{equation}
samples. Of course, here the Big-Theta notation (\(\Theta\)) is hiding the constant involved.

While the Central Limit Theorem (because it is only an asymptotic result) does not yield an ras, it does give a rough estimate of the number of samples required. A CLT approximation indicates that about \(2 (1 - p)p^{-1} \epsilon^{-2} \ln(2 \delta^{-1})\) samples should be necessary and sufficient to obtain an \((\epsilon, \delta)\)-ras.

This is a problem, because the number of samples needed has a factor of \((1 - p) / p\), and \(p\) is the very thing being estimated! Dagum, Karp, Luby, and Ross \cite{Hu_DaKaLuRo2000} found an elegant solution to this difficulty through an algorithm that will be called \textsf{DKLR} throughout this work.

Their idea was to use the Bernoulli samples to create iid geometric random variables. Say that \(G \sim \textsf{Geo}(p)\) if for all \(i \in \{1, 2, \ldots \}\), \({\mathbb{P}}(G = i) = p(1 - p)^{i - 1}\). Then it is straightforward to see that for
\begin{equation}
R = \inf\{r: B_1 + \cdots + B_r = 1\},
\end{equation}
then \(R \sim \textsf{Geo}(p)\).

The mean of \(R\) is \(1 / p\), and the variance is \((1 - p) / p^2\). Hence for the sample average of \(k\) iid draws \(R_1, \ldots, R_k\) with the same distribution as \(R\),
\begin{equation}
\frac{R_1 + \cdots + R_k}{k}
\end{equation}
has mean \(1 / p\) and standard deviation \((1 - p)^{1/2} / [k^{1/2} p]\).

That is, the standard deviation is already of the same order as the mean being estimated! The \textsf{DKLR} algorithm then runs as follows.

\hypertarget{section}{%
\paragraph*{\texorpdfstring{\textsf{DKLR}}{}}\label{section}}
\addcontentsline{toc}{paragraph}{\textsf{DKLR}}

\begin{enumerate}
\def\labelenumi{\arabic{enumi}.}
\item
  Given \(\epsilon\) and \(\delta\), set \(k\) equal to
  \[
   \lceil 1 + (1 + \epsilon)4(e - 2) \epsilon^{-2} \ln(2 \delta^{-1}) \rceil.
   \]
\item
  Use the Bernoulli stream to generate \(R_1, \ldots, R_k\) iid \(\textsf{Geo}(p)\).
\item
  Return
  \[
   \hat p_{\textsf{DKLR}} = \frac{k - 1}{R_1 + \cdots + R_k}.
   \]
\end{enumerate}

The choice of \(k\) in the first line comes from \cite{Hu_DaKaLuRo2000}, and is based on their analysis of the value needed to obtain an \((\epsilon, \delta)\)-ras. The use of \(k - 1\) in the numerator of the final line is a change from their algorithm, which used \(k\). The \(k - 1\) used here makes this algorithm more easily comparable to the improvement on the algorithm, \textsf{GBAS}, that is described later on in this section.

While this was the first ras for this problem, there are some drawbacks. First, the constant of \(4(e - 2) \approx 2.873\ldots\) is not close to the constant of 2 predicted by the CLT. Second, it does not take advantage of second order effects, that is, the tails of the geometric random variable decline in a similar fashion to normals, allowing for second order reductions in the run time that are significant. Third, it is biased.

To solve these issues, the author introduced the \emph{Gamma Bernoulli Approximation Scheme} (\textsf{GBAS}) \cite{Hu_Hu2017}. This algorithm took the reduction of Bernoullis to Geometrics one step further. An easy to show fact about Geometric random variables is the following. Suppose that \(A\) is an exponentially distributed random variable with rate \(\lambda > 0\) so that for all \(a > 0\), it holds that \({\mathbb{P}}(A > a) = \exp(-\lambda a)\). Write \(A \sim \textsf{Exp}(\lambda)\).

\begin{theorem}
\protect\hypertarget{thm:unnamed-chunk-1}{}\label{thm:unnamed-chunk-1}If \(G \sim \textsf{Geo}(p)\) and \(A_1, \ldots, A_G\) are iid exponential random variables of rate 1, then \(A_1 + \cdots + A_G \sim \textsf{Exp}(p)\).
\end{theorem}

(See, for instance, \cite{Hu_Hu2017} for a proof.)

For \(S \sim \textsf{Exp}(p)\), the mean of \(S\) is \(1 / p\) and standard deviation is also \(1 / p\). Moreover, for \(S_1, \ldots, S_k\) iid \(\textsf{Exp}(p)\), \(M = S_1 + \cdots + S_k\) has a gamma distribution with shape parameter \(k\) and rate parameter \(p\). Write \(M \sim \textsf{Gamma}(k, p)\).

The density of both \(M\) and \(1 / M\) (which has an inverse gamma distribution) are known precisely, and the mean of \(1 / M\) is \(p / (k - 1)\). So
\begin{equation}
\hat p_{\textsf{GBAS}} = (k - 1) / M
\end{equation}
is an unbiased estimator for \(p\).

In addition, gamma random variables are scalable. Multiplying a gamma distributed random variable by a constant results in a new gamma distributed random variable where the new rate constant is the old constant divided by the multiplication factor. Hence
\begin{equation}
\frac{p}{\hat p_{\textsf{GBAS}}} = \frac{p M}{k - 1} \sim \textsf{Gamma}(k, k - 1).
\end{equation}
That is to say, the relative error of the \textsf{GBAS} estimate does not depend on the value \(p\) that is being estimated! This makes it a straightforward numerical calculation to find the minimum value of \(k\) such that for \(X \sim \textsf{Gamma}(k, k - 1)\),
\begin{equation}
{\mathbb{P}}\left(\left|\frac{1}{X} - 1 \right| > \epsilon\right) \leq \delta.
\end{equation}

This allows the \textsf{GBAS} algorithm to immediately give an \((\epsilon, \delta)\)-ras.

\hypertarget{section-1}{%
\paragraph*{\texorpdfstring{\textsf{GBAS}}{}}\label{section-1}}
\addcontentsline{toc}{paragraph}{\textsf{GBAS}}

\begin{enumerate}
  \item Given \( \epsilon \) and \( \delta \), set \( k \) equal to the smallest value such that a \( \textsf{Gamma}(k, k - 1) \) random variable lies outside the interval \( [1 / (1 + \epsilon), 1 / (1 - \epsilon)] \) with probability at most \( \delta \).

  \item Use the Bernoulli stream to generate \( S_1, \ldots, S_k \) iid \( \textsf{Exp}(p) \).

  \item
Return
\[
\hat p_{\textsf{GBAS}} = \frac{k - 1}{S_1 + \cdots + S_k}.
  \]
  \end{enumerate}

This corrects most of the three deficiencies noted earlier for the \textsf{DKLR} algorithm. First, the value of \(k\) can be shown to be at most \(2 \epsilon^{-2} \ln(2 \delta^{-1})\) for \(\epsilon\) and \(\delta\) both less than 1. This follows directly from Theorem 1 of \cite{Hu_FeHuRu2019}. Second, the tails of a gamma decline much like those of a normal, and so the second order effects are retained. Third, the result is unbiased.

Unfortunately, one new deficiency has been added. For a stream of \(R_i\) that are iid \(\textsf{Geo}(p)\), the mean of \(R_1 + \cdots + R_k\) is \(1 / p\) while the standard deviation is the slightly smaller \((1 - p)^{1 / 2} / p\). For \(S_i\) iid \(\textsf{Exp}(p)\) the mean is \(1 / p\) while the standard deviation is exactly \(1 / p\).

From a Monte Carlo point of view, this means that the \textsf{DKLR} algorithm (if \(p\) was known ahead of time) should take about \(1 - p\) times the number of samples that the \textsf{GBAS} algorithm does. While this does not affect much if \(p\) is close to 0, for \(p\) in the middle of the interval \([0, 1]\) this can have a significant effect on the running time. Of course, \(p\) is not known ahead of time, and so there has been no way of getting this \(1 - p\) factor in the run time exactly.

This paper for the first time presents a way of running \textsf{DKLR} close to optimally in order to regain a factor that is asymptotically close for small \(\epsilon\) and \(\delta\) to the \(1 - p\) factor in the run time. The rest of the paper is organized as follows. The next section introduces the two-stage method of running and proves the veracity of the estimate. The following section covers how to gain most of the advantage of the \(1 - p\) factor while still obtaining an unbiased estimate.
The last section then concludes.

\hypertarget{the-two-stage-algorithm}{%
\section{The two-stage algorithm}\label{the-two-stage-algorithm}}

Suppose that \(U_1, U_2, \ldots\) forms an iid sequence of random variables that are uniform over \([0, 1]\). Suppose \(B_i = {\mathbbm{1}}(U_i \leq p)\), where \({\mathbbm{1}}\) is the indicator function that evaluates to 0 if the argument is false and 1 if it is true. Then the \(B_1, B_2, \ldots\) form an iid Bernoulli sequence with mean \(p\). Let
\[
T_k(p) = \inf\{t: {\mathbbm{1}}(U_1 < p) + {\mathbbm{1}}(U_2 < p) + \cdots {\mathbbm{1}}(U_t < p) = k \}.
\]

That is, \(T_k(p)\) is the smallest time \(t\) such that in the first \(t\) members of the Bernoulli iid sequence with mean \(p\), there were exactly \(k\) of the \({\mathbbm{1}}(U_i < p)\) values equal to 1 and the rest are equal to 0. In particular, \(T_1(p) \sim \textsf{Geo}(p)\). Then \(T_k\) is the sum of \(k\) independent geometric random variables, and so has a \emph{negative binomial} distribution with parameters \(k\) and \(p\). Write \(T_k \sim \textsf{NegBin}(k, p)\).

\begin{theorem}
Let \( p \leq p' \), \( k \in \{1, 2, \ldots \} \), \( T_k(p) \sim \textsf{NegBin}(k, p) \) and \( T_k(p') \sim \textsf{NegBin}(k, p') \).  Then for all \( a \), \( {\mathbb{P}}(T_k(p) \leq a) \leq {\mathbb{P}}(T_k(p') \leq a) \).
\end{theorem}

\begin{proof}
  Let \( U_1, U_2, \ldots \) be iid uniforms over \( [0, 1] \) used to form \( T_k(p) \) and \( T_k(p') \).  

  Then the Bernoulli random variables \( {\mathbbm{1}}(U_i < p) \) are monotonically increasing in \( p \), so \( T_k(p) \) is monotonically decreasing in \( p \).  That is, for \( p \leq p' \), \( T_k(p) \geq T_k(p') \).  This means that for all \( a \), \( {\mathbb{P}}(T_k(p) \leq a) \leq {\mathbb{P}}(T_k(p') \leq a) \).

This holds for negative binomials coupled using the same uniforms, but the theorem statement is a statement about the marginal probabilities, which holds or fails regardless of the coupling.  Since it holds for one coupling, it holds for all.
\end{proof}

What does this do for us? In the \textsf{DKLR} algorithm, let \(T_k(p) = R_1 + \cdots + R_k\). Then failure occurs when \(\hat p_{\textsf{DKLR}} = k / T_k(p)\) is either too large or too small. Suppose that
\[
p_1 \leq p \leq p_2
\]
where \(p\) is unknown to the user but \(p_1\) and \(p_2\) are both known. Then the probability that \(T_k(p_1)\) is too large is an upper bound on the probability that \(T_k(p)\) is too large for all \(p \in [p_1, p_2]\). Similarly, the probability that \(T_k(p_2)\) is too small is an upper bound on the probability that \(T_k(p)\) is too small for all \(p \in [p_1, p_2]\).

To bound the upper tail for relative error, note that
\begin{eqnarray}
  {\mathbb{P}}\left(\frac{k - 1}{T_k(p)} > (1 + \epsilon) p\right) & = &
    {\mathbb{P}}\left( T_k(p) < \frac{k - 1}{(1 + \epsilon) p} \right) \\
  & \leq &
    {\mathbb{P}}\left( T_k(p) < \frac{k - 1}{(1 + \epsilon) p_1} \right) \\
  & \leq &
    {\mathbb{P}}\left( T_k(p_2) < \frac{k - 1}{(1 + \epsilon) p_1} \right) 
\end{eqnarray}
Similarly,
\begin{equation}
  {\mathbb{P}}\left(\frac{k - 1}{T_k(p)} < (1 - \epsilon) p\right) \leq 
    {\mathbb{P}}\left( T_k(p_1) > \frac{k - 1}{(1 - \epsilon) p_2} \right) 
\end{equation}

At this point the upper and lower tails for the estimate with \(p\) (the size of which are unknown) have been bounded in terms of an upper tail for \(p_1\) and a lower tail for \(p_2\). These bounds can be used to find a value for \(k\) to use in \textsf{DKLR} that gives an \((\epsilon, \delta)\)-ras for all \(p \in [p_1, p_2]\).

\hypertarget{section-2}{%
\paragraph*{\texorpdfstring{\textsf{find\_k}}{}}\label{section-2}}
\addcontentsline{toc}{paragraph}{\textsf{find\_k}}

\begin{enumerate}

  \item  Given \( \epsilon \), \( \delta \), and \( a < b \), return the smallest value of \( k \) such that 
    \[
    {\mathbb{P}}\left(T_k(a) > \frac{k - 1}{ (1 - \epsilon)b } \right) + {\mathbb{P}}\left(T_k(b) < \frac{k - 1}{(1 + \epsilon) a} \right) \leq \delta.
    \]
    
  \end{enumerate}

The output of this algorithm is, for all \(p \in [a, b]\), a value of \(k\) that yields an \((\epsilon, \delta)\)-ras using \textsf{DKLR}. The problem is, if the length of the interval \(b - a\) is very large, the value of \(k\) this gives could be very much larger than optimal. To solve this issue, break \([a, b]\) into smaller intervals and find \(k\) for each one of the smaller intervals. Then the largest of these \(k\) values found will work for all intervals taken individually, and hence works also for the union of the smaller intervals, which is just \([a, b]\).

In particular, if all that is known is that \(p \in [a, 1]\), then it is possible to partition the large interval \([a, 1]\) into \(m\) small intervals using \(a_i\) satisfying
\[
a = a_0 < a_1 < a_2 < \cdots < a_m = 1.
\]

The probability that the \(T\) is too large or too small for \(p \in [a, 1]\) is bounded above by the maximum value that \(T\) is too large or too small for each of the subintervals \([a_i, a_{i + 1}]\). This leads to the following algorithm.

\hypertarget{section-3}{%
\paragraph*{\texorpdfstring{\textsf{two\_stage\_dklr}}{}}\label{section-3}}
\addcontentsline{toc}{paragraph}{\textsf{two\_stage\_dklr}}

\begin{enumerate}
\def\labelenumi{\arabic{enumi}.}
\item
  Given \(\epsilon\) and \(\delta\), run \textsf{GBAS} with \(\sqrt{\epsilon}\) and \(\delta / 2\). This gives an estimate \(\hat p_1\). Then for \(a = \hat p_1 / (1 + \epsilon)\), \({\mathbb{P}}(p < a) \leq \delta / 2\).
\item
  Break the interval \([a, 1]\) into subintervals of width at most \(\epsilon(1 - a) / 10\).
\item
  For each subinterval, calculate the smallest value of \(k\) (using \textsf{find\_k)} that yields for \(\epsilon\) error at most \(\delta / 2\).
\item
  Draw \(R_1, \ldots, R_k\) iid \(\textsf{Geo}(p)\) using the iid stream of Bernoulli random variables.
\item
  Return \((k - 1) / (R_1 + \cdots + R_k)\).
\end{enumerate}

From the previous results, the following holds.

\begin{theorem}
\protect\hypertarget{thm:unnamed-chunk-2}{}\label{thm:unnamed-chunk-2}Let \(\hat p\) be the output of the algorithm \textsf{two\_stage\_dklr}. Then \(\hat p\) satisfies
\begin{equation}
    {\mathbb{P}}(|(\hat p / p) - 1| > \epsilon) \leq \delta).
  \end{equation}
\end{theorem}

The choice of \(\epsilon(1 - a) / 10\) for the subinterval width ensures that the \(\epsilon\) relative error is much larger than the widths, and so errors arising from the width choice are negligible. The choice to use 10 here is not important, larger values like 100 are still valid, but will take longer to compute the error probabilities.

\hypertarget{tilting}{%
\subsection{Tilting}\label{tilting}}

In \cite{Hu_FeHuRu2019}, the idea of \emph{tilting} the resulting estimate was introduced. Tilting takes advantage of the fact that one tail of the sum of geometrics is unbounded and goes down slightly more slowly, while the other is bounded by \(k\) and so goes down slightly more quickly. By adjusting the estimate by a multiplicative factor that is close to 1, the asymptotic behavior of the tails can be balanced to go down as rapidly as possible.

In the case of both \textsf{DKLR} and \textsf{GBAS}, balancing the tails requires dividing the estimate from earlier by
\[
c_{\text{tilt}} = \frac{2 \epsilon}{(1 - \epsilon^2)\ln(1 + 2 \epsilon / (1 - \epsilon))}.
\]
This is very close to 1. More precisely, \(c_{\text{tilt}} = 1 + (2/3)\epsilon^2 + O(\epsilon^4)\).

\hypertarget{running-time}{%
\subsection{Running time}\label{running-time}}

How does the \textsf{two\_stage\_dklr} with tilting perform in practice? Generally speaking, it is better than both the original \textsf{DKLR} and \textsf{GBAS}, especially when \(p\) is close to 1.

The running time of all these algorithms is determined by \(k / p\). The following table shows what value of \(k\) is needed for the first and second stages in \textsf{two\_stage\_dklr}. The first stage sets the error at \(\epsilon^{-1/2}\), and since the \textsf{GBAS} method uses \(\Theta(\epsilon^{-2})\) samples, the first stage uses approximately
\begin{equation}
  (\epsilon^{1/2})^{-2} = \epsilon^{-1} / (1 - p)
\end{equation}
times the overall optimal number of samples.

Now consider the second stage. The second stage runs slower the smaller the estimate \(\hat p_1\) is. The smallest it can be (assuming stage one was succesful) is \((1 - \sqrt{\epsilon}) p\). It is divided by \(1 + \sqrt{\epsilon}\) to get a lower bound on \(p\). Taken together, the gap between 1 and the lower bound on \(p\) is
\begin{equation}
1 - p(1 - \sqrt{\epsilon}) / (1 + \sqrt{\epsilon}).
\end{equation}

Therefore, this factor of samples is gained, but \(1 / (1 - p)\) is still lost. So stage two uses about
\begin{equation}
\frac{1 - p(1 - \sqrt{\epsilon}) / (1 + \sqrt{\epsilon}) + \epsilon}{1 - p}
\end{equation}
times the optimal number of \textsf{DKLR} samples, which itself has a speedup of \(1 / (1 - p)\) over \textsf{GBAS}.

A final effect to consider is that \textsf{GBAS} has \(\delta\) as the error bound, while each stage of two stage \textsf{DKLR} has \(\delta / 2\) as the error bound so that the total probability of error is at most \(\delta\) by the union bound. Combining this effect with the earlier calculations, the speedup from \textsf{GBAS} to two stage \textsf{DKLR} is about
\begin{equation}
\rho = \frac{\ln(\delta^{-1})}{\ln(2) + \ln(\delta^{-1})} \cdot \frac{1}{1 - p(1 - \sqrt{\epsilon}) / (1 + \sqrt{\epsilon}) + \epsilon}.
\end{equation}

This is just an estimate of the speedup, but as Table \ref{Hu_tab1} shows, this estimate is fairly accurate. As \(\epsilon\) and \(\delta\) approach zero, this estimate \(\rho\) of the speedup approaches the optimal speedup of \(1 / ( 1 - p)\).

\begin{table}
  \caption{Determination of \( k \) for various values of \( (p, \epsilon, \delta) \).  The speedup is the k for \textsf{GBAS} divided by the sum of the \( k \) values for Stage 1 and Stage 2.  The absolute best this could be for any algorithm is \( 1 / (1 - p) \).  The estimate of the speedup for the two stage algorithm is \( \rho = \ln(\delta^{-1}) / [(2 + \ln(\delta^{-1})) / (1 - p(1 - \sqrt{\epsilon}) / (1 + \sqrt{\epsilon}) + \epsilon)] \).}
  \label{Hu_tab1}
  \begin{tabular}{p{0.54cm}p{0.54cm}p{1.2cm}p{1.3cm}p{1.54cm}p{1.54cm}p{1.3cm}p{1.3cm}p{1.54cm}}
\toprule
\( p \) & \( \epsilon \) & \( \delta \) & \( k \) \textsf{GBAS} & \( k \) Stage 1 & \( k \) Stage 2 & Speedup & \( \rho \) & \( 1 / (1 - p) \) \\
\midrule
0.9 & 0.10 & 1e-02 & 661 & 76 & 413 & 1.35 & 1.37 & 10.00\\
0.9 & 0.10 & 1e-06 & 2380 & 239 & 1317 & 1.53 & 1.51 & 10.00\\
0.9 & 0.01 & 1e-06 & 239268 & 2513 & 66203 & 3.48 & 3.48 & 10.00\\
0.5 & 0.10 & 1e-02 & 661 & 76 & 551 & 1.05 & 1.03 & 2.00\\
0.5 & 0.10 & 1e-06 & 2380 & 239 & 1760 & 1.19 & 1.13 & 2.00\\
0.5 & 0.01 & 1e-06 & 239268 & 2513 & 145055 & 1.62 & 1.58 & 2.00\\
0.1 & 0.10 & 1e-02 & 661 & 76 & 595 & 0.99 & 0.83 & 1.11\\
0.1 & 0.10 & 1e-06 & 2380 & 239 & 1901 & 1.11 & 0.91 & 1.11\\
0.1 & 0.01 & 1e-06 & 239268 & 2513 & 191853 & 1.23 & 1.03 & 1.11\\
\bottomrule
\end{tabular}
\end{table}

\hypertarget{unbased-approach}{%
\section{Unbased approach}\label{unbased-approach}}

While \textsf{two\_stage\_dklr} gives the fewest number of samples, it does suffer from the fact that the estimate is biased. To deal with this issue while maintaining most of the benefit from the \(1 - p\) factor, a \emph{shifted grid approach} can be used.

For positive integer \(n\), consider a grid of points \(x = (0, 1 / n, 2 / n, \ldots, (n - 1) / n)\). If \(U\) is uniform over \([0, 1]\), then \(x + U \text{ mod } 1\) consists of a sets of points that are not independent, but which are identically distributed. That is, each point will be uniform over \([0, 1]\). As usual, mod 1 means that for any value in \([1, 2)\), 1 is subtracted from the value. For \(U \sim \textsf{Unif}([0, 1])\), set
\begin{equation}
    \label{Hu_shifted_grid}
    (U_1, \ldots, U_n) =\left(\frac{0 + U}{n}, \frac{1 + U}{n}, \frac{2 + U}{n}, \ldots \frac{n - 1 + U}{n} \right) \text{ mod } 1.
  \end{equation}

The \textsf{GBAS} method first uses the Bernoulli values to draw \(M\) as the sum of \(k\) iid geometric random variables with mean \(1 / p\). Then given \(M\), draw \(G\) as gamma with shape parameter \(M\) and rate parameter 1. The result is a gamma distribution with parameters \(k\) and \(p\).

Given a uniform \(U\) over \([0, 1]\), the \emph{inverse transform method} (ITM) can be used to draw from a given one dimensional distribution. For a random variable \(X\) with cumulative distribution function \(F_X\), the pseudoinverse is defined as
\begin{equation}
F_X^{-1}(a) = \inf\{b: F_X(b) \geq a \}.
\end{equation}

The ITM method relies on the fact that for \(U\) uniform over \([0, 1]\), \(X\) and \(F_X^{-1}(U)\) have the same distribution.

This implies that for all \(i \in \{1, \ldots, n \}\),
\begin{equation}
\hat p_i = \frac{k - 1}{F_{[G \mid M]}^{-1}(U_i)}
\end{equation}
is an unbiased estimator of \(p\), and the sample average of the \(\hat p_i\) will be as well.

This sample average \(\hat p_i\) tends to be close to \(1 / M\) (which is the \textsf{DKLR} estimate), as seen next.

\subsection{Non-interval centered unbiased estimate}

By using the new \textsf{DKLR} method, a confidence interval of width at most \(2\epsilon p\) is created with the biased estimate at its center. In the process, a negative binomial \(M\) with parameters \(k\) and \(p\) is created. At the same time, using the shifted grid method an unbiased estimator can also be found at the same time from \(M\).

The unbiased estimator will most likely be very close to the estimate \((k - 1) / M\). This is because the relative error between \([G | M] \sim \textsf{Gamma}(M, 1)\) and \(M\) is typically small. Let
\begin{equation}
  \label{Hu_eqn_biased_estimate}
  \hat p_{\text{b}} = (k - 1) / M
\end{equation}
be the biased estimate and for \(U_i\) drawn using \eqref{Hu_shifted_grid},
\begin{equation}
  \label{Hu_unbiased_estimate}
  \hat p_{\text{unb}} = (k - 1) \frac{1}{n} \sum_{i = 1}^n 1 / F^{-1}_{[G \mid M]}(U_i)
\end{equation}
be the unbiased estimate.

In calculating \(\hat p_{\text{unb}} / \hat p_{\text{b}}\), the \(k - 1\) factor cancels out, leaving us with
\begin{equation}
\label{HU_ratio}
r = \frac{\hat p_{\text{unb}}}{\hat p_{\text{b}}} = M \frac{1}{n} \sum_{i = 1}^n 1 / F^{-1}_{[G \mid M]}(U_i).
\end{equation}

When \(r = 1\), the unbiased and biased entries are the same, so \(|r - 1|\) measures the relative error between the two. With probability \(\delta_1\), \(U \in [\delta_1 / 2, 1 - \delta_1 / 2]\). Plugging these interval endpoints for \(U\) into \eqref{HU_ratio} gives an upper bound on \(|r - 1|\).

Table \ref{Hu_shifted_grid_rel_error} shows the bounds on the distribution of \(r\) for some other values of the variables. As seen in this table, the chance of even a small relative difference between the biased and unbiased estimates becomes very small very fast as \(M\) grows. It should also be noted that for \(n = 10^3\), the unbiased estimate can be computed in a fraction of a second.

\begin{table}[!t]
\caption{Bounds on the relative error between the unbiased and biased estimate.  Here \( r \) is the unbiased estimate divided by the biased estimate.}
\label{Hu_shifted_grid_rel_error} 
\begin{tabular}{p{2cm}p{2.4cm}p{2cm}p{5.4cm}}
\toprule
  \( M \) & \( n \) & \( \delta_1 \) & \( x \) such that \( {\mathbb{P}}(|r - 1| \leq x) = 1 - \delta_1 \) \\
\midrule
\( 10^4 \) & \( 10^3 \) & \( 10^{-6} \) & \( 0.0001497 \) \\
\( 10^4 \) & \( 10^4 \) & \( 10^{-6} \) & \( 0.0001049 \) \\
\( 10^4 \) & \( 10^3 \) & \( 10^{-8} \) & \( 0.0001597 \) \\
\( 10^4 \) & \( 10^2 \) & \( 10^{-8} \) & \( 0.0006899 \) \\
\( 10^5 \) & \( 10^3 \) & \( 10^{-8} \) & \( 0.0000283 \) \\
\bottomrule
\end{tabular}
\end{table}

\hypertarget{conclusion}{%
\section{Conclusion}\label{conclusion}}

Although the \textsf{DKLR} algorithm provides an \((\epsilon, \delta)\)-ras for estimating the mean of a Bernoulli random variable from a stream of such variables, it has not been widely used. This is partially due to loose bounds on the number of samples needed. The best speedup over existing methods that could be achieved is \(1 / (1 - p)\). This work gives a new version of the algorithm that uses a number of samples that approaches the optimal speedup. Roughly, the speedup for this two stage algorithm is
\[
\frac{\ln(\delta^{-1})}{\ln(2) + \ln(\delta^{-1})} \cdot \frac{1}{1 - p(1 - \sqrt{\epsilon}) / (1 + \sqrt{\epsilon}) + \epsilon}.
\]

By using a randomly shifted grid and the inverse transform method, this approach can also yield an unbiased estimate that is very close to the biased estimate. Therefore the interval surrounding the biased estimate can also be used to give an interval around the unbiased estimate of only slightly smaller size.

\bibliographystyle{plain}
\bibliography{refs}

\appendix

\hypertarget{code-for-the-paper}{%
\section{Code for the paper}\label{code-for-the-paper}}

The following code produces the tables in the main text. First load in the needed libraries.

\begin{Shaded}
\begin{Highlighting}[]
\FunctionTok{library}\NormalTok{(kableExtra)}
\FunctionTok{library}\NormalTok{(tidyverse)}
\end{Highlighting}
\end{Shaded}

\hypertarget{preliminary-functions}{%
\subsection{Preliminary functions}\label{preliminary-functions}}

The first function, \texttt{tilt\_value}, calculates the amount to divide the unbiased estimate in \(\textsf{GBAS}\) to make the tails go down at the same rate.

\begin{Shaded}
\begin{Highlighting}[]
\NormalTok{tilt\_value }\OtherTok{\textless{}{-}} \ControlFlowTok{function}\NormalTok{(epsilon) }
  \DecValTok{2} \SpecialCharTok{*}\NormalTok{ epsilon }\SpecialCharTok{/}\NormalTok{ (}\DecValTok{1} \SpecialCharTok{{-}}\NormalTok{ epsilon}\SpecialCharTok{\^{}}\DecValTok{2}\NormalTok{) }\SpecialCharTok{/} \FunctionTok{log}\NormalTok{(}\DecValTok{1} \SpecialCharTok{+} \DecValTok{2} \SpecialCharTok{*}\NormalTok{ epsilon }\SpecialCharTok{/} 
\NormalTok{                                        (}\DecValTok{1} \SpecialCharTok{{-}}\NormalTok{ epsilon))}
\end{Highlighting}
\end{Shaded}

This next function calculates the minimum possible \(k\) needed for a given \(\epsilon\) and \(\delta\) for \(\textsf{GBAS}\). Note that once the function has been run once, the parameter \texttt{min\_k} can be set appropriately to make it run more quickly on a second run.

\begin{Shaded}
\begin{Highlighting}[]
\NormalTok{find\_k\_GBAS }\OtherTok{\textless{}{-}} \ControlFlowTok{function}\NormalTok{(epsilon, delta, }\AttributeTok{tilt =} \ConstantTok{FALSE}\NormalTok{, }\AttributeTok{min\_k =} \DecValTok{1}\NormalTok{) \{}
\NormalTok{  tilt\_constant }\OtherTok{\textless{}{-}} \DecValTok{1} \SpecialCharTok{{-}}\NormalTok{ tilt }\SpecialCharTok{+} \FunctionTok{tilt\_value}\NormalTok{(epsilon) }\SpecialCharTok{*}\NormalTok{ tilt}
\NormalTok{  k }\OtherTok{\textless{}{-}}\NormalTok{ min\_k}
\NormalTok{  accept }\OtherTok{\textless{}{-}} \ConstantTok{FALSE}
  \ControlFlowTok{while}\NormalTok{ (}\SpecialCharTok{!}\NormalTok{accept) \{}
\NormalTok{    k }\OtherTok{\textless{}{-}}\NormalTok{ k }\SpecialCharTok{+} \DecValTok{1}
\NormalTok{    high\_error }\OtherTok{\textless{}{-}} \DecValTok{1} \SpecialCharTok{{-}} \FunctionTok{pgamma}\NormalTok{(}\DecValTok{1} \SpecialCharTok{/}\NormalTok{ (}\DecValTok{1} \SpecialCharTok{{-}}\NormalTok{ epsilon), }
                             \AttributeTok{shape =}\NormalTok{ k, }
                             \AttributeTok{rate =}\NormalTok{ (k }\SpecialCharTok{{-}} \DecValTok{1}\NormalTok{) }\SpecialCharTok{/}\NormalTok{ tilt\_constant)}
\NormalTok{    low\_error }\OtherTok{\textless{}{-}} \FunctionTok{pgamma}\NormalTok{(}\DecValTok{1} \SpecialCharTok{/}\NormalTok{ (}\DecValTok{1} \SpecialCharTok{+}\NormalTok{ epsilon), }
                        \AttributeTok{shape =}\NormalTok{ k, }
                        \AttributeTok{rate =}\NormalTok{ (k }\SpecialCharTok{{-}} \DecValTok{1}\NormalTok{) }\SpecialCharTok{/}\NormalTok{ tilt\_constant)}
\NormalTok{    accept }\OtherTok{\textless{}{-}}\NormalTok{ (high\_error }\SpecialCharTok{+}\NormalTok{ low\_error) }\SpecialCharTok{\textless{}}\NormalTok{ delta}
\NormalTok{  \}}
  \FunctionTok{return}\NormalTok{(k)}
\NormalTok{\}}
\end{Highlighting}
\end{Shaded}

It is helpful to have a vector version of this function.

\begin{Shaded}
\begin{Highlighting}[]
\NormalTok{find\_k\_GBAS\_vector }\OtherTok{\textless{}{-}} \ControlFlowTok{function}\NormalTok{(epsilon, delta, }\AttributeTok{tilt =} \ConstantTok{FALSE}\NormalTok{, }\AttributeTok{min\_k =} \DecValTok{1}\NormalTok{) \{}
\NormalTok{  A }\OtherTok{\textless{}{-}} \FunctionTok{cbind}\NormalTok{(epsilon, delta, tilt, min\_k)}
  \FunctionTok{return}\NormalTok{(}\FunctionTok{apply}\NormalTok{(A, }\DecValTok{1}\NormalTok{, }\ControlFlowTok{function}\NormalTok{(x) }\FunctionTok{find\_k\_GBAS}\NormalTok{(x[}\DecValTok{1}\NormalTok{], x[}\DecValTok{2}\NormalTok{], x[}\DecValTok{3}\NormalTok{], x[}\DecValTok{4}\NormalTok{])))}
\NormalTok{\}}
\end{Highlighting}
\end{Shaded}

For a minimum value of \(p\), the \texttt{find\_error\_dklr\_small\_num} function finds the largest probability of the \(\textsf{DKLR}\) estimate being too small and too large, and returns the results as a length two vector.

\begin{Shaded}
\begin{Highlighting}[]
\NormalTok{find\_error\_dklr\_small\_num }\OtherTok{\textless{}{-}} \ControlFlowTok{function}\NormalTok{(minp, epsilon, k, }\AttributeTok{tilt =} \ConstantTok{FALSE}\NormalTok{) \{}
\NormalTok{  change\_p }\OtherTok{\textless{}{-}}\NormalTok{ (}\DecValTok{1} \SpecialCharTok{{-}}\NormalTok{ minp) }\SpecialCharTok{*}\NormalTok{ epsilon }\SpecialCharTok{/} \DecValTok{100}
\NormalTok{  tilt }\OtherTok{\textless{}{-}} \FunctionTok{tilt\_value}\NormalTok{(epsilon)}
\NormalTok{  seqp }\OtherTok{\textless{}{-}} \FunctionTok{seq}\NormalTok{(minp, }\DecValTok{1}\NormalTok{, }\AttributeTok{by =}\NormalTok{ change\_p)}
\NormalTok{  lower }\OtherTok{\textless{}{-}}\NormalTok{ seqp[}\SpecialCharTok{{-}}\FunctionTok{length}\NormalTok{(seqp)]}
\NormalTok{  upper }\OtherTok{\textless{}{-}}\NormalTok{ seqp[}\SpecialCharTok{{-}}\DecValTok{1}\NormalTok{]}
\NormalTok{  lower\_error }\OtherTok{\textless{}{-}} \DecValTok{1} \SpecialCharTok{{-}} \FunctionTok{pnbinom}\NormalTok{((k }\SpecialCharTok{{-}} \DecValTok{1}\NormalTok{) }\SpecialCharTok{/}\NormalTok{ (lower }\SpecialCharTok{*}\NormalTok{ (}\DecValTok{1} \SpecialCharTok{{-}}\NormalTok{ epsilon)) }\SpecialCharTok{/}\NormalTok{ tilt }\SpecialCharTok{{-}}\NormalTok{ k, k, upper)}
\NormalTok{  upper\_error }\OtherTok{\textless{}{-}} \FunctionTok{pnbinom}\NormalTok{((k }\SpecialCharTok{{-}} \DecValTok{1}\NormalTok{) }\SpecialCharTok{/}\NormalTok{ (upper }\SpecialCharTok{*}\NormalTok{ (}\DecValTok{1} \SpecialCharTok{+}\NormalTok{ epsilon)) }\SpecialCharTok{/}\NormalTok{ tilt }\SpecialCharTok{{-}}\NormalTok{ k, k, lower)}
\NormalTok{  A }\OtherTok{\textless{}{-}} \FunctionTok{cbind}\NormalTok{(seqp[}\SpecialCharTok{{-}}\FunctionTok{length}\NormalTok{(seqp)], }
\NormalTok{             seqp[}\SpecialCharTok{{-}}\DecValTok{1}\NormalTok{], }
\NormalTok{             lower\_error,}
\NormalTok{             upper\_error)}
  \FunctionTok{return}\NormalTok{(}\FunctionTok{c}\NormalTok{(}\FunctionTok{max}\NormalTok{(A[,}\DecValTok{3}\NormalTok{], }\AttributeTok{na.rm =} \ConstantTok{TRUE}\NormalTok{), }\FunctionTok{max}\NormalTok{(A[,}\DecValTok{4}\NormalTok{], }\AttributeTok{na.rm =} \ConstantTok{TRUE}\NormalTok{)))}
\NormalTok{\}}
\end{Highlighting}
\end{Shaded}

\begin{Shaded}
\begin{Highlighting}[]
\NormalTok{find\_k\_DKLR }\OtherTok{\textless{}{-}} \ControlFlowTok{function}\NormalTok{(min\_p, epsilon, delta, }\AttributeTok{tilt =} \ConstantTok{FALSE}\NormalTok{, }\AttributeTok{min\_k =} \DecValTok{1}\NormalTok{) \{}
\NormalTok{  tilt\_constant }\OtherTok{\textless{}{-}} \DecValTok{1} \SpecialCharTok{{-}}\NormalTok{ tilt }\SpecialCharTok{+} \FunctionTok{tilt\_value}\NormalTok{(epsilon) }\SpecialCharTok{*}\NormalTok{ tilt}
\NormalTok{  k }\OtherTok{\textless{}{-}}\NormalTok{ min\_k}
\NormalTok{  accept }\OtherTok{\textless{}{-}} \ConstantTok{FALSE}
  \ControlFlowTok{while}\NormalTok{ (}\SpecialCharTok{!}\NormalTok{accept) \{}
\NormalTok{    k }\OtherTok{\textless{}{-}}\NormalTok{ k }\SpecialCharTok{+} \DecValTok{1}
\NormalTok{    error }\OtherTok{\textless{}{-}} \FunctionTok{sum}\NormalTok{(}\FunctionTok{find\_error\_dklr\_small\_num}\NormalTok{(min\_p, epsilon, k, tilt))}
\NormalTok{    accept }\OtherTok{\textless{}{-}}\NormalTok{ error }\SpecialCharTok{\textless{}}\NormalTok{ delta}
    \ControlFlowTok{if}\NormalTok{ ((k }\SpecialCharTok{{-}}\NormalTok{ min\_k) }\SpecialCharTok{\%\%} \DecValTok{1000} \SpecialCharTok{==} \DecValTok{0}\NormalTok{) }\FunctionTok{print}\NormalTok{(k)}
\NormalTok{  \}}
  \FunctionTok{return}\NormalTok{(k)}
\NormalTok{\}}
\end{Highlighting}
\end{Shaded}

It is helpful to have a version of this function that can take vectors as input and return a vector of outputs.

\begin{Shaded}
\begin{Highlighting}[]
\NormalTok{find\_k\_DKLR\_vector }\OtherTok{\textless{}{-}} 
  \ControlFlowTok{function}\NormalTok{(min\_p, epsilon, delta, }\AttributeTok{tilt =} \ConstantTok{FALSE}\NormalTok{, }\AttributeTok{min\_k =} \DecValTok{1}\NormalTok{) }
  \FunctionTok{apply}\NormalTok{(}\FunctionTok{cbind}\NormalTok{(min\_p, epsilon, delta, tilt, min\_k), }\DecValTok{1}\NormalTok{, }
        \ControlFlowTok{function}\NormalTok{(x) }\FunctionTok{find\_k\_DKLR}\NormalTok{(x[}\DecValTok{1}\NormalTok{], x[}\DecValTok{2}\NormalTok{], x[}\DecValTok{3}\NormalTok{], x[}\DecValTok{4}\NormalTok{], x[}\DecValTok{5}\NormalTok{]))}
\end{Highlighting}
\end{Shaded}

\hypertarget{making-table-1}{%
\subsection{Making Table 1}\label{making-table-1}}

Set up the values needed for Table 1.

\begin{Shaded}
\begin{Highlighting}[]
\NormalTok{p }\OtherTok{\textless{}{-}} \FunctionTok{c}\NormalTok{(}\FunctionTok{rep}\NormalTok{(}\FloatTok{0.9}\NormalTok{, }\DecValTok{3}\NormalTok{), }\FunctionTok{rep}\NormalTok{(}\FloatTok{0.5}\NormalTok{, }\DecValTok{3}\NormalTok{), }\FunctionTok{rep}\NormalTok{(}\FloatTok{0.1}\NormalTok{, }\DecValTok{3}\NormalTok{))}
\NormalTok{epsilon }\OtherTok{\textless{}{-}} \FunctionTok{rep}\NormalTok{(}\FunctionTok{c}\NormalTok{(}\FloatTok{0.1}\NormalTok{, }\FloatTok{0.1}\NormalTok{, }\FloatTok{0.01}\NormalTok{), }\DecValTok{3}\NormalTok{)}
\NormalTok{delta }\OtherTok{\textless{}{-}} \FunctionTok{rep}\NormalTok{(}\FunctionTok{c}\NormalTok{(}\DecValTok{10}\SpecialCharTok{\^{}}\NormalTok{(}\SpecialCharTok{{-}}\DecValTok{2}\NormalTok{), }\DecValTok{10}\SpecialCharTok{\^{}}\NormalTok{(}\SpecialCharTok{{-}}\DecValTok{6}\NormalTok{), }\DecValTok{10}\SpecialCharTok{\^{}}\NormalTok{(}\SpecialCharTok{{-}}\DecValTok{6}\NormalTok{)), }\DecValTok{3}\NormalTok{)}
\end{Highlighting}
\end{Shaded}

Reset the variable \texttt{k\_min1} and \texttt{k\_min2} after the first run to speed up subsequent runs of the code.

\begin{Shaded}
\begin{Highlighting}[]
\CommentTok{\# Set up for GBAS}
\NormalTok{k\_gbas }\OtherTok{\textless{}{-}} \FunctionTok{rep}\NormalTok{(}\DecValTok{1}\NormalTok{, }\FunctionTok{length}\NormalTok{(p))}

\CommentTok{\# Set up for stage one}
\NormalTok{k\_min1 }\OtherTok{\textless{}{-}} \FunctionTok{rep}\NormalTok{(}\DecValTok{1}\NormalTok{, }\FunctionTok{length}\NormalTok{(p))}
\NormalTok{k\_min1 }\OtherTok{\textless{}{-}} \FunctionTok{c}\NormalTok{(}\DecValTok{74}\NormalTok{, }\DecValTok{237}\NormalTok{, }\DecValTok{2511}\NormalTok{, }\DecValTok{74}\NormalTok{, }\DecValTok{237}\NormalTok{, }\DecValTok{2511}\NormalTok{, }\DecValTok{74}\NormalTok{, }\DecValTok{237}\NormalTok{, }\DecValTok{2511}\NormalTok{)}

\CommentTok{\# Set up for stage two}
\NormalTok{k\_min2 }\OtherTok{\textless{}{-}} \FunctionTok{rep}\NormalTok{(}\DecValTok{1}\NormalTok{, }\FunctionTok{length}\NormalTok{(p))}
\NormalTok{k\_min2 }\OtherTok{\textless{}{-}} \FunctionTok{c}\NormalTok{(}\DecValTok{411}\NormalTok{, }\DecValTok{1315}\NormalTok{, }\DecValTok{66200}\NormalTok{, }\DecValTok{549}\NormalTok{, }\DecValTok{1758}\NormalTok{, }\DecValTok{145053}\NormalTok{, }\DecValTok{593}\NormalTok{, }\DecValTok{1899}\NormalTok{, }\DecValTok{191851}\NormalTok{)}

\NormalTok{table1 }\OtherTok{\textless{}{-}} \FunctionTok{tibble}\NormalTok{(p, epsilon, delta, }\AttributeTok{tilt =} \ConstantTok{TRUE}\NormalTok{, k\_min1) }\SpecialCharTok{|\textgreater{}}
  \FunctionTok{mutate}\NormalTok{(}\AttributeTok{sq\_ep =} \FunctionTok{sqrt}\NormalTok{(epsilon),}
         \AttributeTok{ha\_delta =}\NormalTok{ delta }\SpecialCharTok{/} \DecValTok{2}\NormalTok{,}
         \AttributeTok{low\_p =}\NormalTok{ p }\SpecialCharTok{*}\NormalTok{ (}\DecValTok{1} \SpecialCharTok{{-}}\NormalTok{ sq\_ep) }\SpecialCharTok{/}\NormalTok{ (}\DecValTok{1} \SpecialCharTok{+}\NormalTok{ sq\_ep)) }\SpecialCharTok{|\textgreater{}}
  \FunctionTok{mutate}\NormalTok{(}\StringTok{\textasciigrave{}}\AttributeTok{k GBAS}\StringTok{\textasciigrave{}} \OtherTok{=} 
         \FunctionTok{find\_k\_GBAS\_vector}\NormalTok{(epsilon, delta, tilt, k\_gbas), }
         \AttributeTok{.before =} \StringTok{\textasciigrave{}}\AttributeTok{tilt}\StringTok{\textasciigrave{}}\NormalTok{) }\SpecialCharTok{|\textgreater{}}
  \FunctionTok{mutate}\NormalTok{(}\StringTok{\textasciigrave{}}\AttributeTok{k Stage 1}\StringTok{\textasciigrave{}} \OtherTok{=} 
         \FunctionTok{find\_k\_GBAS\_vector}\NormalTok{(sq\_ep, ha\_delta, tilt, k\_min1), }
         \AttributeTok{.before =} \StringTok{\textasciigrave{}}\AttributeTok{tilt}\StringTok{\textasciigrave{}}\NormalTok{) }\SpecialCharTok{|\textgreater{}}
  \FunctionTok{mutate}\NormalTok{(}\StringTok{\textasciigrave{}}\AttributeTok{k Stage 2}\StringTok{\textasciigrave{}} \OtherTok{=} 
         \FunctionTok{find\_k\_DKLR\_vector}\NormalTok{(low\_p, epsilon, ha\_delta, }\ConstantTok{TRUE}\NormalTok{, k\_min2),}
         \AttributeTok{.before =}\NormalTok{ tilt) }\SpecialCharTok{|\textgreater{}}
  \FunctionTok{mutate}\NormalTok{(}\AttributeTok{speedup =} \FunctionTok{round}\NormalTok{(}\StringTok{\textasciigrave{}}\AttributeTok{k GBAS}\StringTok{\textasciigrave{}} \SpecialCharTok{/}\NormalTok{ (}\StringTok{\textasciigrave{}}\AttributeTok{k Stage 1}\StringTok{\textasciigrave{}} \SpecialCharTok{+} \StringTok{\textasciigrave{}}\AttributeTok{k Stage 2}\StringTok{\textasciigrave{}}\NormalTok{), }\DecValTok{2}\NormalTok{),}
         \AttributeTok{rho =} \FunctionTok{round}\NormalTok{(}
           \FunctionTok{log}\NormalTok{(delta}\SpecialCharTok{\^{}}\NormalTok{(}\SpecialCharTok{{-}}\DecValTok{1}\NormalTok{)) }\SpecialCharTok{/}\NormalTok{ (}\FunctionTok{log}\NormalTok{(}\DecValTok{2} \SpecialCharTok{*}\NormalTok{ delta}\SpecialCharTok{\^{}}\NormalTok{(}\SpecialCharTok{{-}}\DecValTok{1}\NormalTok{))) }\SpecialCharTok{/} 
\NormalTok{           (}\DecValTok{1} \SpecialCharTok{{-}}\NormalTok{ p }\SpecialCharTok{*}\NormalTok{ (}\DecValTok{1} \SpecialCharTok{{-}}\NormalTok{ sq\_ep) }\SpecialCharTok{/}\NormalTok{ (}\DecValTok{1} \SpecialCharTok{+}\NormalTok{ sq\_ep) }\SpecialCharTok{+}\NormalTok{ epsilon), }\DecValTok{2}\NormalTok{),}
         \AttributeTok{geo\_speed =} \FunctionTok{round}\NormalTok{(}\DecValTok{1} \SpecialCharTok{/}\NormalTok{ (}\DecValTok{1} \SpecialCharTok{{-}}\NormalTok{ p), }\DecValTok{2}\NormalTok{),}
         \AttributeTok{.before =}\NormalTok{ tilt) }\SpecialCharTok{|\textgreater{}}
  \FunctionTok{select}\NormalTok{(}\DecValTok{1}\SpecialCharTok{:}\DecValTok{9}\NormalTok{)}
\end{Highlighting}
\end{Shaded}

\begin{Shaded}
\begin{Highlighting}[]
\NormalTok{table1 }\SpecialCharTok{|\textgreater{}} 
  \FunctionTok{kable}\NormalTok{(}\AttributeTok{booktabs =} \ConstantTok{TRUE}\NormalTok{) }\SpecialCharTok{|\textgreater{}} 
  \FunctionTok{kable\_styling}\NormalTok{(}\AttributeTok{latex\_options =} \FunctionTok{c}\NormalTok{(}\StringTok{"scale\_down"}\NormalTok{))}
\end{Highlighting}
\end{Shaded}

\begin{table}
\centering
\resizebox{\linewidth}{!}{
\begin{tabular}{rrrrrrrrr}
\toprule
p & epsilon & delta & k GBAS & k Stage 1 & k Stage 2 & speedup & rho & geo\_speed\\
\midrule
0.9 & 0.10 & 1e-02 & 661 & 76 & 413 & 1.35 & 1.37 & 10.00\\
0.9 & 0.10 & 1e-06 & 2380 & 239 & 1317 & 1.53 & 1.51 & 10.00\\
0.9 & 0.01 & 1e-06 & 239268 & 2513 & 66203 & 3.48 & 3.48 & 10.00\\
0.5 & 0.10 & 1e-02 & 661 & 76 & 551 & 1.05 & 1.03 & 2.00\\
0.5 & 0.10 & 1e-06 & 2380 & 239 & 1760 & 1.19 & 1.13 & 2.00\\
\addlinespace
0.5 & 0.01 & 1e-06 & 239268 & 2513 & 145055 & 1.62 & 1.58 & 2.00\\
0.1 & 0.10 & 1e-02 & 661 & 76 & 595 & 0.99 & 0.83 & 1.11\\
0.1 & 0.10 & 1e-06 & 2380 & 239 & 1901 & 1.11 & 0.91 & 1.11\\
0.1 & 0.01 & 1e-06 & 239268 & 2513 & 191853 & 1.23 & 1.03 & 1.11\\
\bottomrule
\end{tabular}}
\end{table}

\hypertarget{making-table-2}{%
\subsection{Making Table 2}\label{making-table-2}}

Given a value of \(M\), \(n\), and \(\delta_1\), find the maximum relative error between the unbiased and the biased estimate.

\begin{Shaded}
\begin{Highlighting}[]
\NormalTok{find\_error\_shifted\_grid }\OtherTok{\textless{}{-}} \ControlFlowTok{function}\NormalTok{(M, n , delta1) \{}
  \CommentTok{\# create a shifted grid on [0, 1] with n elements}
\NormalTok{  grid }\OtherTok{\textless{}{-}} \FunctionTok{seq}\NormalTok{(}\DecValTok{0}\NormalTok{, }\DecValTok{1} \SpecialCharTok{{-}} \DecValTok{1} \SpecialCharTok{/}\NormalTok{ n , }\AttributeTok{by =} \DecValTok{1} \SpecialCharTok{/}\NormalTok{ n)}
  \CommentTok{\# Do the 0 part first}
\NormalTok{  lower }\OtherTok{\textless{}{-}} \FunctionTok{abs}\NormalTok{(}\DecValTok{1} \SpecialCharTok{{-}} 
\NormalTok{    M }\SpecialCharTok{*} \FunctionTok{mean}\NormalTok{(}\DecValTok{1} \SpecialCharTok{/} \FunctionTok{qgamma}\NormalTok{(grid }\SpecialCharTok{+}\NormalTok{ delta1 }\SpecialCharTok{/} \DecValTok{2}\NormalTok{, M, }\DecValTok{1}\NormalTok{)))}
  \CommentTok{\# Do the high part second}
\NormalTok{  high }\OtherTok{\textless{}{-}} \FunctionTok{abs}\NormalTok{(}\DecValTok{1} \SpecialCharTok{{-}} 
\NormalTok{    M }\SpecialCharTok{*} \FunctionTok{mean}\NormalTok{(}\DecValTok{1} \SpecialCharTok{/} \FunctionTok{qgamma}\NormalTok{(grid }\SpecialCharTok{+} \DecValTok{1} \SpecialCharTok{/}\NormalTok{ n }\SpecialCharTok{{-}}\NormalTok{ delta1 }\SpecialCharTok{/} \DecValTok{2}\NormalTok{, M, }\DecValTok{1}\NormalTok{)))}
  \FunctionTok{return}\NormalTok{(}\FunctionTok{max}\NormalTok{(lower, high))}
\NormalTok{\}}
\end{Highlighting}
\end{Shaded}

\begin{Shaded}
\begin{Highlighting}[]
\NormalTok{find\_error\_shifted\_grid\_vector }\OtherTok{\textless{}{-}} \ControlFlowTok{function}\NormalTok{(M, n , delta1) }
  \FunctionTok{apply}\NormalTok{(}\FunctionTok{cbind}\NormalTok{(M, n, delta1), }\DecValTok{1}\NormalTok{, }
        \ControlFlowTok{function}\NormalTok{ (x) }\FunctionTok{find\_error\_shifted\_grid}\NormalTok{(}
\NormalTok{                      x[}\DecValTok{1}\NormalTok{], x[}\DecValTok{2}\NormalTok{], x[}\DecValTok{3}\NormalTok{]))}
\end{Highlighting}
\end{Shaded}

\begin{Shaded}
\begin{Highlighting}[]
\NormalTok{table2 }\OtherTok{\textless{}{-}}
  \FunctionTok{tibble}\NormalTok{(}
    \AttributeTok{M =} \FunctionTok{c}\NormalTok{(}\DecValTok{10}\SpecialCharTok{\^{}}\DecValTok{4}\NormalTok{, }\DecValTok{10}\SpecialCharTok{\^{}}\DecValTok{4}\NormalTok{, }\DecValTok{10}\SpecialCharTok{\^{}}\DecValTok{4}\NormalTok{, }\DecValTok{10}\SpecialCharTok{\^{}}\DecValTok{4}\NormalTok{, }\DecValTok{10}\SpecialCharTok{\^{}}\DecValTok{5}\NormalTok{),}
    \AttributeTok{n =} \FunctionTok{c}\NormalTok{(}\DecValTok{10}\SpecialCharTok{\^{}}\DecValTok{3}\NormalTok{, }\DecValTok{10}\SpecialCharTok{\^{}}\DecValTok{4}\NormalTok{, }\DecValTok{10}\SpecialCharTok{\^{}}\DecValTok{3}\NormalTok{, }\DecValTok{10}\SpecialCharTok{\^{}}\DecValTok{2}\NormalTok{, }\DecValTok{10}\SpecialCharTok{\^{}}\DecValTok{3}\NormalTok{),}
    \AttributeTok{delta1 =} \FunctionTok{c}\NormalTok{(}\DecValTok{10}\SpecialCharTok{\^{}}\NormalTok{(}\SpecialCharTok{{-}}\DecValTok{6}\NormalTok{), }\DecValTok{10}\SpecialCharTok{\^{}}\NormalTok{(}\SpecialCharTok{{-}}\DecValTok{6}\NormalTok{), }\DecValTok{10}\SpecialCharTok{\^{}}\NormalTok{(}\SpecialCharTok{{-}}\DecValTok{8}\NormalTok{), }\DecValTok{10}\SpecialCharTok{\^{}}\NormalTok{(}\SpecialCharTok{{-}}\DecValTok{8}\NormalTok{), }\DecValTok{10}\SpecialCharTok{\^{}}\NormalTok{(}\SpecialCharTok{{-}}\DecValTok{8}\NormalTok{)),}
    \AttributeTok{final =} \FunctionTok{find\_error\_shifted\_grid\_vector}\NormalTok{(M, n, delta1)}
\NormalTok{  )}
\end{Highlighting}
\end{Shaded}

\begin{Shaded}
\begin{Highlighting}[]
\NormalTok{table2 }\SpecialCharTok{|\textgreater{}} 
  \FunctionTok{kable}\NormalTok{(}\AttributeTok{booktabs =} \ConstantTok{TRUE}\NormalTok{, }\AttributeTok{digits =} \DecValTok{8}\NormalTok{) }\SpecialCharTok{|\textgreater{}} \FunctionTok{kable\_styling}\NormalTok{()}
\end{Highlighting}
\end{Shaded}

\begin{table}
\centering
\begin{tabular}{rrrr}
\toprule
M & n & delta1 & final\\
\midrule
1e+04 & 1000 & 1e-06 & 0.00014967\\
1e+04 & 10000 & 1e-06 & 0.00010491\\
1e+04 & 1000 & 1e-08 & 0.00015871\\
1e+04 & 100 & 1e-08 & 0.00068990\\
1e+05 & 1000 & 1e-08 & 0.00002826\\
\bottomrule
\end{tabular}
\end{table}

\end{document}